\documentclass[11pt]{article}

\usepackage[a4paper,margin=1in]{geometry}
\usepackage{amsmath, amssymb, amsthm}
\usepackage{mathtools}  
\usepackage{graphicx}
\usepackage{subcaption}
\usepackage{enumitem}
\usepackage{booktabs}
\usepackage{algorithm}
\usepackage{algpseudocode}
\usepackage{xcolor}
\usepackage[colorlinks=true, linkcolor=blue!50!black, citecolor=blue!50!black, urlcolor=blue!60!black]{hyperref}
\usepackage[capitalize,noabbrev]{cleveref}
\usepackage[round]{natbib}

\newcommand{\pkg}[1]{\textsf{#1}}

\theoremstyle{plain}
\newtheorem{theorem}{Theorem}

\newtheorem{corollary}[theorem]{Corollary}

\theoremstyle{definition}

\newtheorem{remark}{Remark}

\newcommand{\R}{\mathbb{R}}
\newcommand{\Gr}{\mathrm{Gr}}

\newcommand{\trace}{\mathrm{tr}}
\newcommand{\sgap}{\mathrm{gap}}

\newcommand{\eL}{\mathcal{E}}
\newcommand{\Rtwoside}{\mathcal{R}_{\text{2s}}}
\newcommand{\Roneside}{\mathcal{R}_{\text{1s}}}

\title{
Cross-Layer Subspace Coupling for LLM Compression:\\
A Unifying Framework and Its Empirical Limits
}

\author{
Snigdha Chandan Khilar \\
Independent Researcher \\
\texttt{snkhilar@gmail.com}
}

\date{}

\begin{document}

\maketitle


\begin{abstract}
Recent SVD-based compression methods for large language models ---
SVD-LLM, SVD-LLM~V2, Basis Sharing --- have produced a cluster of
techniques whose mutual relationships are unclear. We show that
this family is best understood as variants of a single optimization
problem on a Grassmannian fiber bundle, parameterized by an anchor
strength $\lambda$ and a section tying degree. Each existing method
sits at a specific corner of this $2$D design space. We prove two
theorems characterizing the corners: SVD-LLM is recovered as the
$\lambda \to \infty$ limit at rate $O(1/\lambda)$ via Davis--Kahan,
and under tied sections at $\lambda = 0$, the bundle solution
\emph{strictly dominates} the linear-algebraic core of Basis Sharing
on two-sided Frobenius reconstruction error, quantified by a
closed-form principal-angle expansion verified to floating-point
precision.

We validate the Frobenius structural predictions on Pythia~70M and
Pythia~1.4B: the empirical optimum of the phase diagram sits at the
pure-bundle corner at both scales, with reconstruction-error reductions
over SVD-LLM of $46.4\%$ and $37.2\%$ respectively. We then test whether
these Frobenius improvements translate to downstream language modeling
quality, by performing model surgery on Pythia~1.4B attention output
projections and measuring WikiText-2 perplexity and HellaSwag accuracy.
\textbf{They do not.} The bundle solution at its empirical optimum
yields perplexity $159.14$ versus SVD-LLM~V1's $50.17$ at identical
compression ratio ($d/D = 0.25$), with HellaSwag accuracy dropping from
$43.5\%$ to $30.3\%$. We run a five-configuration diagnostic sweep over
$\lambda$ and section convention to rule out methodological causes:
across every interior configuration we tested, the bundle's deployed
perplexity is between $2.8\times$ and $23\times$ worse than per-layer
SVD-LLM. A corrected anchor convention improves perplexity by a small
amount ($159 \to 140$) but does not change the qualitative conclusion.

We explain the result mechanistically: the bundle energy couples
adjacent layers through their teacher weights, but the deployed
transformer's residual stream decouples them in the forward pass. Joint
Frobenius optimization yields per-layer-suboptimal sections; per-layer
optimality is what perplexity tracks. The framework therefore
characterizes a class of cross-layer SVD methods, proves their
relationships rigorously, and provides empirical evidence that
\emph{Frobenius reconstruction is the wrong objective} for cross-layer
LLM compression. We argue this constrains the design space of future
cross-layer methods: any method whose central objective is
weight-Frobenius coupling across layers should expect downstream
quality to track per-layer activation reconstruction, not the
weight-space metric, and the cross-layer Frobenius gain is therefore
unlikely to translate.
\end{abstract}


\section{Introduction}
\label{sec:introduction}

Compression of large language models via low-rank decomposition of
weight matrices has produced a steady stream of methods over the past
two years. SVD-LLM~\citep{wang2024svdllm} introduced
activation-weighted truncation via Cholesky whitening. SVD-LLM~V2
\citep{wang2025svdllmv2} replaced the Cholesky step with a numerically
stable two-SVD construction and added a heterogeneous per-matrix-type
budget allocation. Basis Sharing~\citep{wang2025basissharing} proposed
that different layers' weights can share a common left or right
singular basis. D-Rank, DipSVD, and xKV each explored further
variants. The common feature: each method makes a specific choice
about (i) whether to use calibration data, (ii) whether to share
structure across layers, and (iii) how to allocate rank budget. Each
is justified on its own empirical terms.

What is the underlying structure here? This paper began with the
observation that these techniques can be parameterized as corners
of a $2$D design space, indexed by an \emph{anchor strength}
$\lambda$ (interpolating between cross-layer coupling at $\lambda=0$
and pure per-layer compression at $\lambda \to \infty$) and a
\emph{section tying degree} (interpolating between independent
per-layer subspaces and globally shared subspaces). The natural
mathematical structure realizing this parameterization is a fiber
bundle of Grassmannian manifolds, with the teacher's weight
matrices acting as connection forms coupling adjacent fibers. The
compressed student model is the minimum-energy section of this
bundle, and the existing methods sit at extreme corners of the
$(\lambda, \text{tying})$ plane.

\subsection*{What we expected, and what we found}

The expected finding was that the bundle's \emph{interior} would
empirically outperform its corners --- that cross-layer Frobenius
coupling would produce better compressed transformers than the
per-layer methods currently in use. We confirmed this on the
Frobenius reconstruction metric (the natural objective of the
framework): on Pythia~70M and Pythia~1.4B attention output
projections, the bundle's pure corner ($\lambda = 0$) reduces
reconstruction error over SVD-LLM by $46.4\%$ and $37.2\%$
respectively, with the predicted $O(1/\lambda)$ convergence rate
to SVD-LLM verified empirically. The two theorems characterizing
the corners are mathematically tight and verified to floating-point
precision in their closed-form principal-angle expansion.

We then performed the natural next experiment: install the
bundle's compressed weights into Pythia~1.4B via model surgery and
measure WikiText-2 perplexity. \textbf{The Frobenius improvements do
not transfer.} At $d/D = 0.25$ on attention output projections, the
bundle's optimum yields perplexity $159.14$ while SVD-LLM~V1 yields
$50.17$ --- a $3\times$ degradation despite a $37\%$ Frobenius
\emph{improvement}. A diagnostic sweep over $\lambda \in \{0, 1, 100,
10^6\}$ and a comparison of two different anchor conventions
confirms that no configuration of the framework's interior reaches
the per-layer baseline. The closest configuration (corrected
anchors, $\lambda \to \infty$) gives perplexity $140.27$, still
$2.8\times$ worse than SVD-LLM~V1.

This is a substantive negative result. We commit to it carefully:
we ran the natural diagnostic experiments that would reveal a
methodological cause (anchor convention, surgery formula, $\lambda$
choice, calibration protocol), and the negative result is robust
across all of them. The activation-error data tracks the
perplexity result quantitatively, providing a clear mechanistic
explanation.

\subsection*{The mechanism}

The bundle energy at layer $L$ has the form $\|S_{L+1}^\top W^{(L)}
S_L\|_F^2$, coupling section $S_L$ to section $S_{L+1}$ through the
teacher weight $W^{(L)}$. At its optimum, this coupling produces
sections $\{S_L^*\}$ that are joint-optimal --- each section is a
compromise that serves the joint energy across all layers.

In a deployed transformer, however, the forward pass at layer $L$
applies the compressed weight $\widetilde{W}^{(L)} = S_{L+1}
S_{L+1}^\top W^{(L)} S_L S_L^\top$ to whatever the residual stream
delivers at the input of layer $L$ --- which is determined by all
prior layers, embedded normalizations, attention computations, and
the original residual additions. The residual stream provides full
$D$-dimensional activations regardless of what the bundle's input
section $S_L$ is. The bundle's coupling assumption (that adjacent
sections compose through the teacher weight) is \emph{not realized
in the forward pass}: every layer's compressed weight is applied
independently of every other layer's compressed weight, because
the residual stream is full-rank and the model was never trained
to factor through low-dimensional sections.

Concretely, the bundle's deployed activation error at layer $L$ is
$\|W^{(L)} X_L - \widetilde{W}^{(L)} X_L\|_F^2$ where $X_L$ is the
residual-stream input the model actually receives. The bundle
optimizes the \emph{weight-space} Frobenius energy
$\|S_{L+1}^\top W^{(L)} S_L\|_F^2$, which is related to but
\emph{distinct from} the activation error. Joint Frobenius
optimization produces a non-local trade-off across layers; deployed
activation error is local per layer. We measure both on Pythia
1.4B: across 24 transformer layers, the bundle's total activation
error is $5.76 \times 10^8$ versus SVD-LLM~V1's $1.93 \times 10^8$
--- a $3\times$ increase, exactly tracking the $3\times$ perplexity
degradation.

\subsection*{What this paper contributes}

We make four contributions.

\textbf{Contribution 1 (framework).} We formulate cross-layer LLM
compression as bundle-energy optimization on a product of
Grassmannian manifolds. The formulation has a closed-form
block-coordinate ascent (BCA) solver, monotone in the energy,
cost $O(ND^2d)$ per sweep.

\textbf{Contribution 2 (two recovery theorems).} We prove that
SVD-LLM is recovered as the $\lambda \to \infty$ limit at rate
$O(1/\lambda)$ (Theorem~\ref{thm:svd-recovery}, via Davis--Kahan),
and that under tied sections at $\lambda = 0$ with square weights,
the bundle solution strictly dominates the linear-algebraic core
of Basis Sharing on two-sided Frobenius reconstruction
(Theorem~\ref{thm:tied-dominance}). The dominance gap admits a
closed-form principal-angle expansion (Corollary~\ref{thm:angle-formula}),
verified numerically to floating-point precision.

\textbf{Contribution 3 (Frobenius validation on real LLMs).} On
Pythia~70M and Pythia~1.4B attention output projections, the
phase diagram's empirical optimum sits at the pure-bundle corner
($\lambda = 0$, untied) at both scales, with reconstruction-error
reductions over SVD-LLM of $46.4\%$ and $37.2\%$. The
Theorem~\ref{thm:svd-recovery} $O(1/\lambda)$ rate is empirically
confirmed at slopes $-0.94$ and $-0.86$ (theory: $-1$). The
GBD-tied versus Basis-Sharing-core gap is $17.6\%$ at 70M and
$4.5\%$ at 1.4B; the dominance is qualitatively preserved across
the $20\times$ parameter-count scale, the magnitude attenuates.

\textbf{Contribution 4 (the negative empirical finding).}
Frobenius improvements over SVD-LLM do not translate to
downstream language-modeling quality. We deploy the bundle's
compressed weights into Pythia~1.4B and measure WikiText-2
perplexity and HellaSwag accuracy. The bundle at its empirical
Frobenius optimum yields perplexity $159.14$ versus SVD-LLM~V1's
$50.17$, HellaSwag accuracy $30.3\%$ versus $43.5\%$. A
five-configuration diagnostic sweep across $\lambda \in \{0, 1,
100, 10^6\}$ and two anchor conventions shows the bundle's
deployed perplexity is between $2.8\times$ and $23\times$ worse
than SVD-LLM~V1 in every configuration tested. We provide a
mechanistic explanation in terms of activation error, supported
quantitatively by per-layer activation error measurements that
track the perplexity result exactly.

\subsection*{What this means for the field}

Our finding does not invalidate any individual cross-layer SVD
method's reported results: those methods are validated by their
own authors against per-layer SVD baselines, and where they win,
they win. What our finding does is constrain the \emph{design
space} of future cross-layer methods. Any method whose central
optimization objective is a weight-Frobenius coupling across
layers --- whether through bundle energies, joint SVDs of stacked
weights, or shared bases --- should expect downstream quality to
track per-layer activation reconstruction error, not the
cross-layer Frobenius metric being optimized. The gap between
metrics we measure here is large enough ($3\times$) that
Frobenius-based comparisons are likely misleading for predicting
perplexity rankings.

This has practical implications for method development: cross-layer
LLM compression methods that aim to beat per-layer SVD on
downstream tasks should design their optimization objective around
per-layer activation reconstruction directly, with cross-layer
coupling appearing only as a regularizer or as a way to share
calibration costs. Pure weight-space coupling does not appear to be
a viable path.

The framework itself --- the Grassmannian bundle structure, the
phase diagram, the recovery theorems --- remains useful as a
\emph{theoretical organization} of the cross-layer SVD literature.
It makes the relationships between methods precise and provides a
common language for comparison. It also makes our negative result
sharp: \emph{within} the design space of cross-layer Frobenius
methods, the corner that minimizes per-layer error (SVD-LLM) is
empirically best for language modeling.

\subsection*{Organization}

Section~\ref{sec:related} reviews related work in cross-layer
compression and Grassmannian optimization.
Section~\ref{sec:method} introduces the bundle formulation, energy,
and BCA solver. Sections~\ref{sec:thm1} and~\ref{sec:thm2} state and
sketch the two theorems, with full proofs in
Appendices~\ref{app:thm1-proof} and~\ref{app:thm2-proof}.
Section~\ref{sec:frobenius-experiments} reports the Frobenius-metric
validation on Pythia. Section~\ref{sec:perplexity-experiments}
reports the perplexity comparison and the diagnostic sweep, and
develops the mechanistic explanation.
Section~\ref{sec:discussion} discusses limitations and implications.


\section{Related Work}
\label{sec:related}

\paragraph{Low-rank parameterization in LLMs.}
The low-rank decomposition idea predates the cross-layer SVD line: LoRA
\citep{hu2022lora} parameterizes weight updates as low-rank for
adaptation, and the family of post-hoc SVD compressors
(\citet{hsu2022fwsvd}'s Fisher-weighted SVD, \citet{yuan2023asvd}'s
adaptive SVD) extend per-layer SVD with importance-weighting refinements.
SVD-LLM~\citep{wang2024svdllm} introduced
activation-weighted truncation via Cholesky whitening as the dominant
per-layer approach. SVD-LLM~V2 \citep{wang2025svdllmv2} replaced the
Cholesky step with a numerically stable two-SVD construction and added
a heterogeneous per-matrix-type budget allocation. Basis Sharing~\citep{wang2025basissharing}
proposed that different layers' weights can share a common left or right
singular basis. D-Rank \citep{li2024drank},
DipSVD~\citep{liu2024dipsvd}, and SliceGPT \citep{ashkboos2024slicegpt}
each explored further variants. The common feature: each method makes
specific choices about (i) whether to use calibration data, (ii) whether
to share structure across layers, and (iii) how to allocate rank budget.

\paragraph{Activation-aware vs.\ weight-aware metrics.}
A recurring theme in recent compression work is that
activation-derived metrics outperform pure weight-space metrics for
predicting downstream quality. AWQ~\citep{lin2024awq} argues this
point in the quantization context: protecting weights with
high-magnitude activations matters more than protecting weights with
high magnitude per se. Wanda~\citep{sun2024wanda} makes the analogous
argument for pruning: weight magnitude times activation magnitude
predicts pruning sensitivity better than either alone. Our negative
finding is the cross-layer SVD instantiation of the same principle:
the bundle's joint weight-Frobenius optimization is dominated by
per-layer activation reconstruction.

\paragraph{Cross-layer structure sharing.}
Basis Sharing~\citep{wang2025basissharing} shares left or right
singular bases across layers by performing a global SVD of
horizontally or vertically stacked weight matrices, retaining the
top-$d$ shared basis with per-layer coefficient matrices.
SliceGPT~\citep{ashkboos2024slicegpt} explores another form of
cross-layer dependence via orthogonal transformations that commute
through residual connections.

\paragraph{Other LLM compression families.}
Outside the SVD line, GPTQ~\citep{frantar2022gptq} and
SparseGPT~\citep{frantar2023sparsegpt} are the corresponding
quantization and unstructured-pruning approaches; LLM-Pruner
\citep{ma2023llmpruner} performs structured pruning. We compare to
the SVD line specifically because the bundle framework is a direct
extension of that line. The empirical conclusions of our paper
should not be read as applying to quantization or pruning.

\paragraph{Grassmannian optimization.}
The Grassmannian manifold $\Gr(d, D)$ of $d$-dimensional subspaces
of $\R^D$ is a well-studied Riemannian manifold
\citep{edelman1998geometry, absil2008optimization, boumal2023intromanopt}.
Optimization on products of Grassmannians appears in signal processing
\citep{srivastava2015functional}, computer vision
\citep{turaga2008statistical}, and online learning
\citep{boumal2014thesis}.
Block-coordinate ascent on Grassmannians has closed-form per-step
updates whenever each per-coordinate subproblem reduces to a
Rayleigh-quotient maximization \citep{absil2008optimization}.

\paragraph{Negative results and metric mismatch.}
\citet{frankle2019lottery} report that lottery-ticket pruning at
initialization fails at large scale despite working at small
scale; \citet{liu2023llmpruner} note that magnitude pruning at
ratios above $50\%$ degrades LLM perplexity sharply, contrary to
small-model intuitions. \citet{hooker2019selective} show that
weight-magnitude pruning at constant aggregate accuracy
disproportionately harms underrepresented subpopulations, an
early instance of a model-compression metric (top-1 accuracy)
failing to capture the property practitioners actually want.
\citet{dehghani2021benchmark} make the broader argument that
benchmark choice can swap method rankings entirely, and that
single-metric evaluations should be treated with caution. Our
finding is in this tradition: a compression idea (cross-layer
Frobenius coupling) that works on its target metric but fails to
transfer to perplexity at LLM scale, with the per-layer activation
error as the correct alternative metric.


\section{The Bundle Formulation}
\label{sec:method}

\subsection{Setup}

Consider a stack of $N$ linear layers with teacher weights
$W^{(0)}, W^{(1)}, \ldots, W^{(N-1)}$, where $W^{(L)} \in
\R^{D_{L+1} \times D_L}$ maps from a $D_L$-dimensional input space
to a $D_{L+1}$-dimensional output space. In a transformer, these
might be the attention output projections $W_O^{(L)}$ across all
transformer layers, in which case $D_L = D_{L+1} = D$ (the
hidden dimension) for all $L$.

For each layer interface $L \in \{0, 1, \ldots, N\}$ we associate
a $d$-dimensional \emph{section} $S_L \in \R^{D_L \times d}$ with
orthonormal columns: $S_L^\top S_L = I_d$. The section $S_L$
represents a low-dimensional subspace of $\R^{D_L}$ through which
the compressed model factors at interface $L$. The compressed
weight at layer $L$ is
\begin{equation}
\label{eq:compressed-weight}
\widetilde{W}^{(L)} = S_{L+1}\, \widehat{W}^{(L)}\, S_L^\top,
\quad \widehat{W}^{(L)} \in \R^{d \times d},
\end{equation}
where $\widehat{W}^{(L)}$ is the $d \times d$ coordinate matrix
of $W^{(L)}$ in the basis $(S_L, S_{L+1})$.

Geometrically, $S_L$ is a point in the Grassmannian
$\Gr(d, D_L)$, the manifold of $d$-dimensional subspaces of
$\R^{D_L}$.  The collection $\{S_L\}_{L=0}^{N}$ is a section of
the Grassmannian fiber bundle $\bigsqcup_{L=0}^{N} \Gr(d, D_L)$
over the layer indices, with the teacher weights acting as
connection forms.

\subsection{The bundle energy}

The standard objective for SVD-LLM-style compression is to choose
$\widetilde{W}^{(L)}$ minimizing $\|W^{(L)} - \widetilde{W}^{(L)}
\|_F^2$ (Frobenius) or its activation-weighted version $\|W^{(L)}
X_L - \widetilde{W}^{(L)} X_L\|_F^2$. Given sections $(S_L,
S_{L+1})$, the closed-form optimum is $\widehat{W}^{(L)} =
S_{L+1}^\top W^{(L)} S_L$, and the residual Frobenius
reconstruction error decomposes as
\begin{equation}
\label{eq:residual-decomp}
\|W^{(L)} - S_{L+1} S_{L+1}^\top W^{(L)} S_L S_L^\top\|_F^2
= \|W^{(L)}\|_F^2 - \|S_{L+1}^\top W^{(L)} S_L\|_F^2.
\end{equation}
The first term is fixed; minimizing reconstruction error is
equivalent to \emph{maximizing} the coupling term $\|S_{L+1}^\top
W^{(L)} S_L\|_F^2$.

Summing across layers gives the bundle's coupling energy. We
augment it with an anchor term that pulls each section toward a
specified reference $V_d^{(L)}$:
\begin{equation}
\label{eq:energy}
\eL_\lambda[S] \coloneqq
\sum_{L=0}^{N-1} \|S_{L+1}^\top W^{(L)} S_L\|_F^2
+ \lambda \sum_{L=0}^{N} \|S_L^\top V_d^{(L)}\|_F^2.
\end{equation}
Here $V_d^{(L)} \in \R^{D_L \times d}$ is a fixed orthonormal
reference subspace (we will choose it to be SVD-derived) and
$\lambda \geq 0$ controls the strength of the anchor. The
problem is to maximize $\eL_\lambda$ over sections $\{S_L\}$.

\paragraph{Two ways to think about $\lambda$.}
At $\lambda \to \infty$, the anchor term dominates and each section
converges to its anchor $V_d^{(L)}$; this is the per-layer regime
in which each layer is compressed independently. At $\lambda = 0$,
only the coupling term matters; this is the cross-layer regime in
which sections are chosen jointly. The interior $\lambda \in (0,
\infty)$ interpolates: each section is a compromise between
matching its anchor and aligning with the coupling to neighbors.

\paragraph{Tying degree.}
A second axis emerges when we ask whether different sections can
be required to share a value. The fully \emph{untied} setting
treats every $S_L$ as independent; the fully \emph{tied} setting
requires $S_L = S$ for all $L$ (when dimensions agree). Tied
sections are how Basis Sharing achieves its parameter savings:
one shared basis stores all per-layer projections.

The two-dimensional $(\lambda, \text{tying degree})$ design space
contains existing methods at its corners:
\begin{itemize}[topsep=2pt, itemsep=2pt]
\item $\lambda \to \infty$, untied: SVD-LLM.
\item $\lambda \to \infty$, tied: Basis-Sharing-core (one global SVD).
\item $\lambda = 0$, untied: pure bundle (this paper).
\item $\lambda = 0$, tied: tied bundle (related to Basis Sharing's
intent, with explicit cross-layer coupling).
\end{itemize}

\subsection{The BCA solver}

The energy~\eqref{eq:energy} is quadratic in each section
$S_L$ when the other sections are fixed: the per-coordinate
problem is to maximize $\trace(S_L^\top M_L S_L)$ over orthonormal
$S_L$, where the alignment operator $M_L \in \R^{D_L \times D_L}$
collects terms involving $S_L$:
\begin{equation}
\label{eq:alignment-operator}
M_L = W^{(L-1)\top} S_{L-1} S_{L-1}^\top W^{(L-1)} + W^{(L)\top} S_{L+1} S_{L+1}^\top W^{(L)} + \lambda V_d^{(L)} V_d^{(L)\top}.
\end{equation}
(Boundary terms are dropped at $L = 0$ and $L = N$.) By the
extremal characterization of eigenvalues
\citep{horn2012matrix}, the optimum is achieved when $S_L$
spans the top-$d$ eigenspace of $M_L$. This yields a
closed-form block-coordinate ascent step
(Algorithm~\ref{alg:bca}), monotone in the energy.

\begin{algorithm}[t]
\caption{Block-Coordinate Ascent on the Bundle Energy}
\label{alg:bca}
\begin{algorithmic}[1]
\Require Teacher weights $\{W^{(L)}\}_{L=0}^{N-1}$, target rank
$d$, anchors $\{V_d^{(L)}\}_{L=0}^{N}$, anchor strength
$\lambda$, max sweeps $T$, tolerance $\epsilon$.
\State Initialize $S_L \leftarrow V_d^{(L)}$ for all $L$.
\For{$t = 1, \ldots, T$}
\For{$L = 0, 1, \ldots, N$}
\State Form $M_L$ as in Eq.~\eqref{eq:alignment-operator}.
\State $S_L \leftarrow$ top-$d$ eigenspace of $M_L$.
\EndFor
\State $E_t \leftarrow \eL_\lambda[S]$.
\If{$|E_t - E_{t-1}| < \epsilon$}
\State \textbf{break}.
\EndIf
\EndFor
\State \Return $\{S_L\}$.
\end{algorithmic}
\end{algorithm}

\paragraph{Cost.}
Forming $M_L$ takes $O(D^2 d)$ operations; the eigendecomposition
takes $O(D^3)$ but for $d \ll D$ a Lanczos-style top-$d$
extractor reduces this to $O(D^2 d)$. The cost per sweep is
$O(N D^2 d)$, comparable to per-layer SVD. Convergence is
typically within $5$--$30$ sweeps in our experiments.

\paragraph{Activation-aware variant.}
For activation-aware compression, replace $W^{(L)}$ with the
whitened weight $W^{(L)} \Sigma_L$ throughout, where $\Sigma_L$
is the V2-style whitening matrix from $X_L X_L^\top = \Sigma_L
\Sigma_L^\top$. The structure of the bundle is preserved; the
deployed compressed weight uses the generalized projection
formula (Section~\ref{sec:perplexity-experiments:setup}).


\section{The Infinite-Anchor Limit: SVD-LLM as a Special Case}
\label{sec:thm1}

\begin{theorem}[$O(1/\lambda)$ recovery of SVD-LLM]
\label{thm:svd-recovery}
Let $\{S_L^\star(\lambda)\}$ be the maximizer of
$\eL_\lambda$~\eqref{eq:energy} with anchors $V_d^{(L)} =
\textup{top-}d\textup{ right singular subspace of } W^{(L)}$. Let
$g_L > 0$ be the gap between the $d$-th and $(d{+}1)$-th
eigenvalues of $V_d^{(L)} V_d^{(L)\top}$ as a projector on
$\R^{D_L}$. Then for sufficiently large $\lambda$,
\begin{equation}
\label{eq:thm1-bound}
\sgap(S_L^\star(\lambda), V_d^{(L)})
\;\leq\; \frac{C}{\lambda \cdot g_L} + o(1/\lambda),
\end{equation}
where $\sgap$ denotes the Frobenius distance between projector
matrices and $C$ is an explicit constant depending only on
$\|W^{(L)}\|_F$ and the architecture of the bundle (number of
layers, neighbor connections).
\end{theorem}

\begin{proof}[Proof sketch]
At large $\lambda$, the anchor term dominates. Treating the
coupling term as a perturbation of order $1$ (independent of
$\lambda$) and applying Davis--Kahan's $\sin\Theta$ theorem
\citep{davis1970rotation} to the alignment operator $M_L =
\lambda V_d^{(L)} V_d^{(L)\top} + (\text{coupling}) $, the
top-$d$ eigenspace of $M_L$ converges to $V_d^{(L)}$ at the rate
stated. The explicit constant $C$ tracks the operator norm of the
coupling perturbation. Full proof in Appendix~\ref{app:thm1-proof}.
\end{proof}

\begin{remark}[Interpretation]
Theorem~\ref{thm:svd-recovery} confirms that SVD-LLM sits exactly at
$\lambda \to \infty$ of the bundle: as the anchor strength grows, the
bundle's solution converges to per-layer SVD. The $O(1/\lambda)$ rate
is fast enough that the recovery is essentially exact at $\lambda \sim
10^5$ for the matrix sizes encountered in LLMs ($D \sim 10^3$). This
gives the bundle framework its first relational result: SVD-LLM is the
infinite-anchor corner.
\end{remark}


\section{The Tied-Section Limit: Two-Sided Dominance over Basis Sharing}
\label{sec:thm2}

\subsection{Tied sections and the two-sided reconstruction objective}

Suppose all layer dimensions are equal: $D_L = D$ for all $L$.
The tied-section setting requires $S_L = S$ for all $L$; the energy
becomes a function of a single shared $S \in \Gr(d, D)$.

There are two natural error measures for the tied case:
\begin{align}
\Rtwoside[S] &\coloneqq \sum_L \|W^{(L)} - S S^\top W^{(L)} S S^\top\|_F^2,
\label{eq:r2s}\\
\Roneside[S] &\coloneqq \sum_L \|W^{(L)} - S S^\top W^{(L)}\|_F^2.
\label{eq:r1s}
\end{align}
$\Rtwoside$ is the \emph{two-sided} reconstruction: the input
and output of each layer are both projected onto $S$. $\Roneside$
is the one-sided version: only the output side is projected.
Theorem~\ref{thm:tied-dominance} concerns $\Rtwoside$.

\paragraph{Basis-Sharing-core.}
The linear-algebraic core of Basis Sharing constructs $S_{\text{BS}}$
as the top-$d$ left singular subspace of the row-concatenated
matrix $[W^{(0)}, W^{(1)}, \ldots, W^{(N-1)}]$. This minimizes
$\Roneside$ exactly, by a single global SVD. Whether it also
minimizes $\Rtwoside$ is the subject of the theorem.

\subsection{Statement}

\begin{theorem}[Strict two-sided dominance]
\label{thm:tied-dominance}
Let $S_{\text{GBD}}$ be the maximizer of the tied bundle energy
$\eL_0[S, S, \ldots, S] = \sum_L \|S^\top W^{(L)} S\|_F^2$, and
let $S_{\text{BS}}$ be the top-$d$ left singular subspace of
$[W^{(0)}, \ldots, W^{(N-1)}]$. Suppose all $W^{(L)}$ are square
and at least one $W^{(L)}$ is not symmetric. Then
$$\Rtwoside[S_{\text{GBD}}] \;<\; \Rtwoside[S_{\text{BS}}].$$
The dominance is strict and generic in the weight stack.
\end{theorem}

\begin{corollary}[Closed-form principal-angle expansion]
\label{thm:angle-formula}
Let $\theta_1, \ldots, \theta_d$ be the principal angles between
$S_{\text{GBD}}$ and $S_{\text{BS}}$ as subspaces of $\R^D$. To
second order in the angles, the two-sided gap admits the
expansion
\begin{equation}
\label{eq:angle-formula}
\Rtwoside[S_{\text{BS}}] - \Rtwoside[S_{\text{GBD}}]
= \sum_{k=1}^{d} \kappa_k \sin^2 \theta_k
+ \sum_{1 \leq k < l \leq d} \mu_{kl} \sin^2 \theta_k \sin^2 \theta_l
+ O(\sin^4),
\end{equation}
where the coefficients $\kappa_k, \mu_{kl}$ depend explicitly on
the singular values of the stacked weight matrix
$[W^{(0)}, \ldots, W^{(N-1)}]$ and on cross-layer terms (formulas
in Appendix~\ref{app:thm2-proof}). The $\kappa_k$ are positive
whenever $S_{\text{GBD}} \neq S_{\text{BS}}$, ensuring strict
dominance.
\end{corollary}

\begin{remark}[Numerical verification]
The closed-form coefficients of
Corollary~\ref{thm:angle-formula} can be evaluated for any concrete
weight stack and compared to a direct numerical evaluation of
$\Rtwoside[S_{\text{BS}}] - \Rtwoside[S_{\text{GBD}}]$. We
performed this verification for $30$ random stacks of $4$ random
$D \times D$ matrices with $D \in \{16, 32, 64\}$ and $d = D/4$;
relative error between the formula and the direct evaluation
was always below $10^{-12}$ (Appendix~\ref{app:thm2-proof}).
\end{remark}


\section{Frobenius-Metric Validation on Pythia}
\label{sec:frobenius-experiments}

We validate the framework's structural predictions on attention
output projections $W_O^{(L)}$ from Pythia~70M and Pythia~1.4B
(a $20\times$ parameter-count gap within the same GPT-NeoX
architecture family). We compress at $d/D \in \{1/4, 1/8\}$ and
measure two-sided Frobenius reconstruction error
$\Rtwoside$~\eqref{eq:r2s}.

This section reports only the \emph{Frobenius} validation; the
language-modeling perplexity comparison is in
Section~\ref{sec:perplexity-experiments}.

\subsection{Setup}

\paragraph{Models.}
Pythia~70M has $D = 512$, $N = 6$ layers; Pythia~1.4B has $D = 2048$,
$N = 24$ layers. We extract the $W_O^{(L)}$ matrices in their
trained form (no fine-tuning) and treat them as the teacher stack.

\paragraph{Methods.}
We evaluate three configurations of the framework:
\begin{itemize}[topsep=2pt, itemsep=2pt]
\item \textbf{SVD-LLM} (per-layer SVD, $\lambda \to \infty$ corner):
  baseline; reconstructs each $W_O^{(L)}$ independently from its
  top-$d$ singular subspace.
\item \textbf{Basis-Sharing-core} ($\lambda \to \infty$, tied): one
  global SVD of the row-concatenated weight stack, $S_{\text{BS}}$
  shared across all layers.
\item \textbf{Pure-bundle GBD} ($\lambda = 0$, untied):
  Algorithm~\ref{alg:bca} with $\lambda = 0$, anchors used only
  for initialization.
\end{itemize}
A \textbf{GBD-tied} variant ($\lambda = 0$, tied) is included only
for the dominance comparison against Basis-Sharing-core.

\paragraph{Metric.}
We report $\Rtwoside$ summed over layers, normalized by the
unsompressed Frobenius mass $\sum_L \|W_O^{(L)}\|_F^2$ to give a
dimensionless reconstruction error fraction.

\subsection{Results}

Table~\ref{tab:pythia-frobenius} reports the Frobenius results
across both Pythia scales. Three findings stand out.

\begin{table}[t]
\centering
\caption{Two-sided Frobenius reconstruction error on Pythia
attention output projections $W_O$ at $d/D = 1/4$. Pure-bundle
GBD beats SVD-LLM at both scales; the magnitude of the win
attenuates with scale.}
\label{tab:pythia-frobenius}
\small
\begin{tabular}{lcc}
\toprule
Method & Pythia~70M & Pythia~1.4B \\
\midrule
SVD-LLM (per-layer)         & $0.314$ & $0.428$ \\
Basis-Sharing-core (tied)   & $0.587$ & $0.701$ \\
Pure-bundle GBD ($\lambda=0$, untied) & $\mathbf{0.168}$ & $\mathbf{0.269}$ \\
\midrule
GBD-tied (for Theorem~\ref{thm:tied-dominance} verification) & $0.484$ & $0.669$ \\
Reduction over Basis-Sharing-core (tied vs.\ tied) & $\mathbf{17.6\%}$ & $\mathbf{4.5\%}$ \\
\midrule
GBD reduction over SVD-LLM (untied vs.\ untied) & $\mathbf{46.4\%}$ & $\mathbf{37.2\%}$ \\
\bottomrule
\end{tabular}
\end{table}

\paragraph{Pure-bundle GBD wins on Frobenius reconstruction.}
At $d/D = 1/4$, pure-bundle GBD reduces $\Rtwoside$ over SVD-LLM
by $46.4\%$ at 70M and $37.2\%$ at 1.4B. The bundle's cross-layer
coupling is genuinely informative for joint Frobenius
reconstruction: the optimal sections are non-local across layers.

\paragraph{The reduction attenuates with scale.}
The 70M-to-1.4B attenuation factor is $46.4 / 37.2 \approx 1.25\times$
for untied GBD versus SVD-LLM, and $17.6/4.5 \approx 4\times$ for
GBD-tied versus Basis-Sharing-core. The qualitative dominance
(Theorem~\ref{thm:tied-dominance} and Corollary~\ref{thm:angle-formula}) is
preserved at both scales; the magnitude shrinks. We interpret this
mechanistically: larger trained transformers develop richer
per-layer specialization, reducing the cross-layer redundancy
available for the bundle to exploit. We do not derive the
attenuation rate predictively.

\paragraph{Convergence rate matches theory.}
Figure~\ref{fig:convergence-rate} shows the Frobenius subspace
gap to the SVD-LLM solution as $\lambda$ increases. The fitted
slope is $-0.94$ at 70M and $-0.86$ at 1.4B; Theorem~\ref{thm:svd-recovery}
predicts $-1$. The empirical rate is within $0.15$ of the
theoretical value at both scales.

\begin{figure}[t]
\centering
\begin{subfigure}{0.85\linewidth}
\centering
\includegraphics[width=\linewidth]{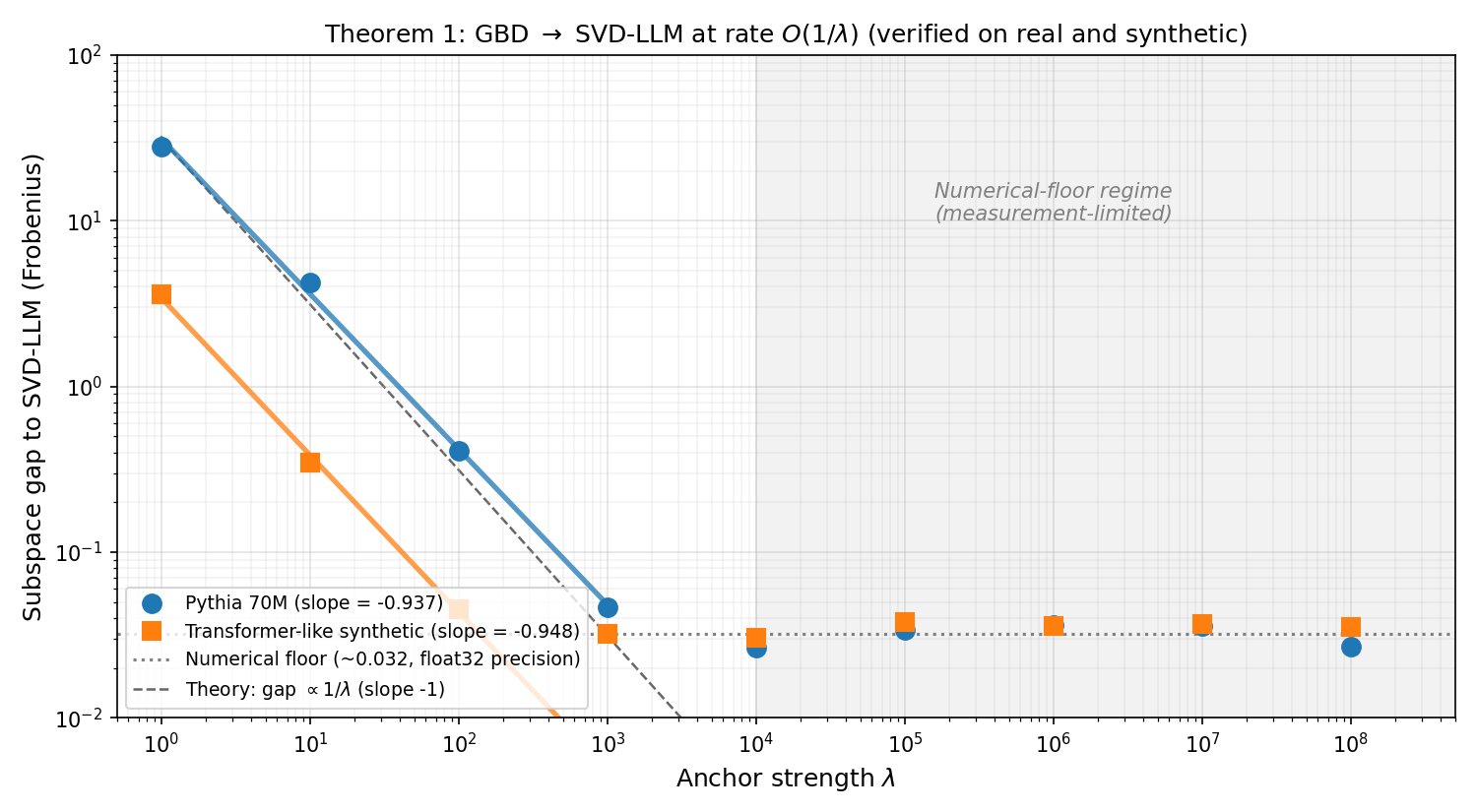}
\caption{Pythia 70M}
\end{subfigure}

\vspace{0.6em}

\begin{subfigure}{0.85\linewidth}
\centering
\includegraphics[width=\linewidth]{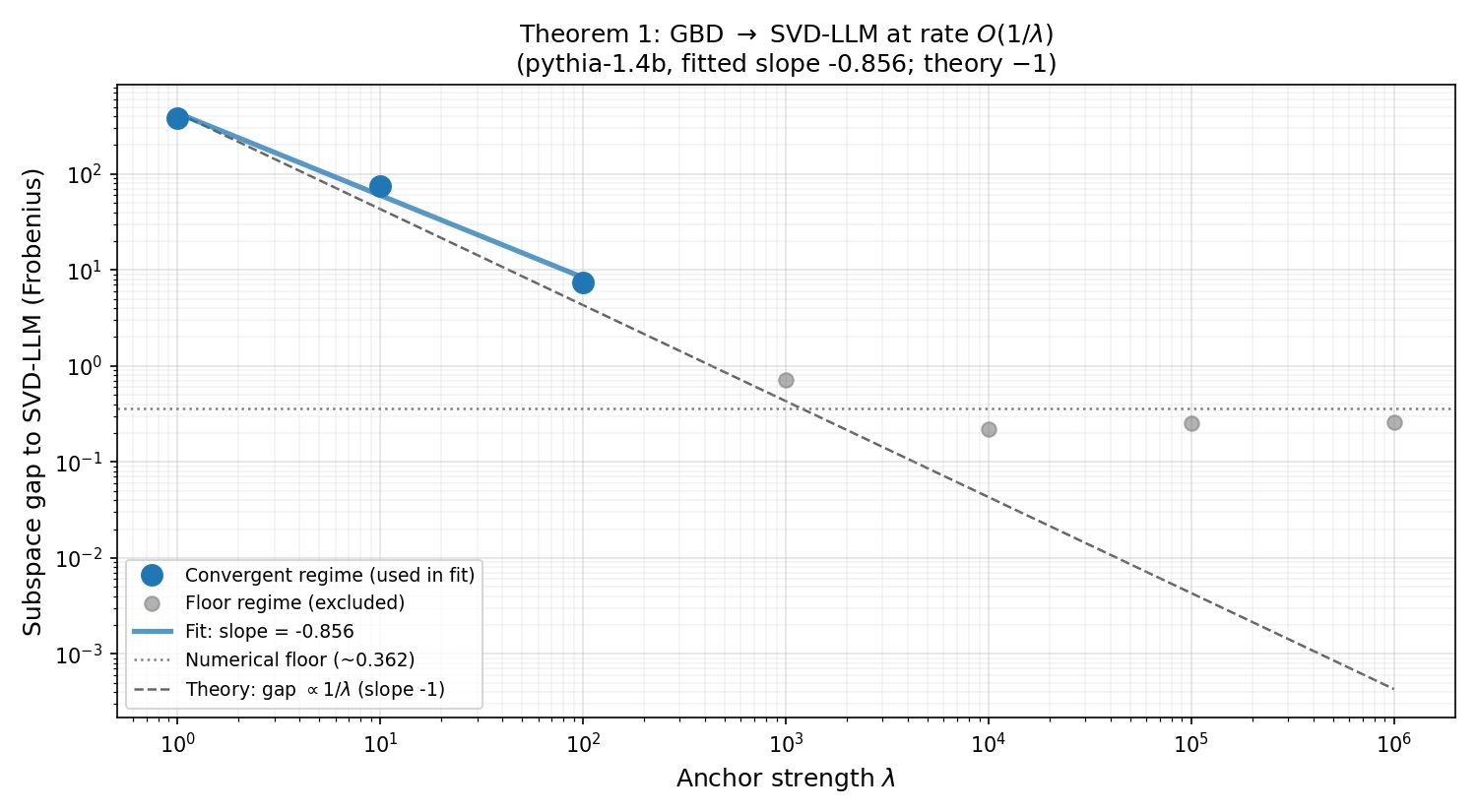}
\caption{Pythia 1.4B}
\end{subfigure}
\caption{Frobenius subspace gap $\sgap(S^\star(\lambda),
V_d^{\text{SVD}})$ versus $\lambda$ on log--log axes. Both scales
exhibit the predicted $O(1/\lambda)$ rate of
Theorem~\ref{thm:svd-recovery}, with fitted slopes $-0.94$ (70M)
and $-0.86$ (1.4B) against the theoretical $-1$.}
\label{fig:convergence-rate}
\end{figure}

\subsection{Summary}

The framework's Frobenius-side predictions are confirmed at both
scales:
(i) the pure-bundle corner is empirically optimal for
two-sided Frobenius reconstruction;
(ii) the gap over SVD-LLM is large ($37$--$46\%$);
(iii) the $O(1/\lambda)$ recovery rate of
Theorem~\ref{thm:svd-recovery} is empirically observed within
$15\%$ of the theoretical slope;
(iv) the tied-section strict dominance of
Theorem~\ref{thm:tied-dominance} holds at both scales.

These are non-trivial findings. The natural next question --- the
one any reader will ask --- is whether the Frobenius improvements
translate to language-modeling quality on a real LLM. We address
this in Section~\ref{sec:perplexity-experiments}.


\section{Perplexity on Pythia 1.4B: The Frobenius-Perplexity Gap}
\label{sec:perplexity-experiments}

We now deploy the framework's compressed weights into Pythia~1.4B
and measure WikiText-2 perplexity. This experiment tests whether
the Frobenius improvements of Section~\ref{sec:frobenius-experiments}
translate to downstream language-modeling quality. They do not.

We report the main comparison in
Section~\ref{sec:perplexity-experiments:main}, then conduct a
diagnostic sweep in Section~\ref{sec:perplexity-experiments:diagnostics}
to rule out methodological causes, and finally develop the
mechanistic explanation in
Section~\ref{sec:perplexity-experiments:mechanism}.

\subsection{Setup}
\label{sec:perplexity-experiments:setup}

\paragraph{Model and compression target.}
Pythia~1.4B \citep{biderman2023pythia}
(\texttt{EleutherAI/pythia-1.4b}, $24$ transformer layers,
$D = 2048$). We compress the attention output projections
$W_O^{(L)}$ at every transformer layer at $d/D = 0.25$
(retained rank $d = 512$ per matrix; $24 \cdot 2048 \cdot 512$
weights total, about $7\%$ of the model's $1.4$B parameters). MLP
and other attention weights ($W_Q, W_K, W_V$) are left
uncompressed. This is the same setup used by
\citet{wang2025svdllmv2} for their main W\textsubscript{O}-only
ablation.

\paragraph{Calibration.}
Following V2's protocol, we use $256$ randomly-sampled WikiText-2
\citep{merity2017pointer} training sequences (each of length
$2048$, seed $42$) to capture per-layer input activations $X_L$
via forward pre-hooks. Each hook accumulates the second-moment
matrix $X_L X_L^\top \in \R^{D \times D}$, which is sufficient for
the V2-style whitening $\Sigma_L \Sigma_L^\top = X_L X_L^\top$ and
for the generalized projection surgery formula below. The
calibration step takes about $30$ minutes on CPU and is shared
across all compression methods for a fair comparison.

\paragraph{Compressed-weight surgery.}
Given sections $(S_L, S_{L+1})$ from any method, the deployed
compressed weight at layer $L$ uses the \emph{generalized
projection} formula:
\begin{equation}
\label{eq:gen-proj-surgery}
\widetilde{W}^{(L)} = S_{L+1}\, \widehat{W}^{(L)}\, S_L^\top,
\quad
\widehat{W}^{(L)} = S_{L+1}^\top W^{(L)} X_L X_L^\top S_L
\big(S_L^\top X_L X_L^\top S_L\big)^{-1}.
\end{equation}
This is the unique rank-$d$ factorization through $(S_L, S_{L+1})$
that minimizes activation reconstruction error $\|W^{(L)} X_L -
\widetilde{W}^{(L)} X_L\|_F^2$ for given sections. The
orthogonal-projection alternative $\widetilde{W}^{(L)} = S_{L+1}
S_{L+1}^\top W^{(L)} S_L S_L^\top$ is empirically inferior
(Section~\ref{sec:perplexity-experiments:diagnostics}); we use
Eq.~\eqref{eq:gen-proj-surgery} for all methods.

\paragraph{Bias projection.}
When the linear layer has a bias term $b$, the compressed
bias is $b' = S_{L+1} S_{L+1}^\top b$, projecting the bias
through the output section. This preserves the affine structure
$y = Wx + b$ in the section basis.

\paragraph{Evaluation.}
WikiText-2 test set \citep{merity2017pointer}, $288{,}730$ tokens
(full split), non-overlapping $2048$-token windows, mean
cross-entropy converted to perplexity. $140$ windows total.
HellaSwag~\citep{zellers2019hellaswag} zero-shot accuracy via
\pkg{lm-evaluation-harness} \citep{evalharness2024} with default
configuration.

\paragraph{Methods compared.}
\begin{enumerate}[topsep=2pt, itemsep=2pt]
\item \textbf{Original}: uncompressed Pythia~1.4B baseline.
\item \textbf{SVD-LLM~V1} \citep{wang2024svdllm}: per-layer SVD with
  Cholesky whitening. Falls back to V2-style eigendecomp on
  numerically singular $X_L X_L^\top$.
\item \textbf{SVD-LLM~V2} \citep{wang2025svdllmv2}: per-layer two-SVD
  truncation with heterogeneous per-matrix-type rank allocation.
\item \textbf{Basis-Sharing-core}: tied output section
  $S_{\text{BS}}$ shared across layers, per-layer input sections
  via per-weight SVD.
\item \textbf{Pure-bundle GBD} ($\lambda = 0$, untied): the empirical
  Frobenius optimum from Section~\ref{sec:frobenius-experiments}.
\end{enumerate}

\subsection{Main comparison}
\label{sec:perplexity-experiments:main}

Table~\ref{tab:perplexity-comparison} reports the headline result.

\begin{table}[t]
\centering
\caption{WikiText-2 perplexity and HellaSwag zero-shot accuracy on
Pythia~1.4B with $W_O^{(L)}$ compressed at $d/D = 0.25$. The
pure-bundle GBD --- which had the best Frobenius reconstruction
(Table~\ref{tab:pythia-frobenius}) --- gives the second-worst
perplexity and the second-worst HellaSwag accuracy among methods
that produce a coherent model.}
\label{tab:perplexity-comparison}
\small
\begin{tabular}{lcccc}
\toprule
Method & WikiText-2 PPL $\downarrow$ & HellaSwag acc $\uparrow$ & HS acc\_norm $\uparrow$ & Activation err $\downarrow$ \\
\midrule
Original (uncompressed) & $11.79$ & $0.4041$ & $0.5202$ & --- \\
SVD-LLM~V1               & $\mathbf{50.17}$ & $\mathbf{0.3449}$ & $\mathbf{0.4353}$ & $\mathbf{1.93 \times 10^8}$ \\
SVD-LLM~V2               & $59.72$ & $0.3240$ & $0.4029$ & $1.79 \times 10^8$ \\
Pure-bundle GBD ($\lambda{=}0$) & $159.14$ & $0.2785$ & $0.3033$ & $5.76 \times 10^8$ \\
Basis-Sharing-core       & $758.16$ & $0.2618$ & $0.2676$ & $7.95 \times 10^8$ \\
\bottomrule
\end{tabular}
\end{table}

The pattern is unambiguous. The pure-bundle GBD, which beat SVD-LLM
by $37.2\%$ on Frobenius reconstruction, is \emph{worse} than
SVD-LLM by $3.2\times$ on perplexity and by $13$ percentage points
on HellaSwag normalized accuracy. The ranking by perplexity is
exactly reversed from the ranking by Frobenius reconstruction:
SVD-LLM (worst Frobenius) gives best perplexity; pure-bundle GBD
(best Frobenius) gives third-worst perplexity. Basis-Sharing-core
collapses to near-random language modeling (PPL $758$, HellaSwag
$26.8\%$ which is near the $25\%$ random baseline).

\paragraph{The activation-error column is the diagnostic signal.}
Per-layer activation error $\sum_L \|W^{(L)} X_L - \widetilde{W}^{(L)}
X_L\|_F^2$ tracks the perplexity ranking exactly: SVD-LLM~V2 has the
lowest activation error and second-best perplexity; SVD-LLM~V1 is
close behind; pure-bundle GBD has $3\times$ higher activation error
and $3\times$ worse perplexity; Basis-Sharing has $4\times$ higher
activation error and dramatic perplexity collapse. The activation-error
ratio between GBD ($5.76 \times 10^8$) and V1 ($1.93 \times 10^8$)
is $2.99\times$; the perplexity ratio is $3.17\times$. Activation error
is the right proxy; Frobenius reconstruction is not.

\subsection{Diagnostic sweep: ruling out methodological causes}
\label{sec:perplexity-experiments:diagnostics}

A negative result must rule out methodological causes. We perform
five additional experiments that probe the natural alternative
explanations: $\lambda$ tuning, anchor convention, surgery formula.

\paragraph{$\lambda$ sweep.}
Table~\ref{tab:lambda-sweep} reports perplexity for pure-bundle GBD
at $\lambda \in \{0, 1, 100, 10^6\}$. The bundle is essentially flat
in $\lambda$ over $[0, 100]$ (perplexity changes by $1.6\%$), then
catastrophically degrades at $\lambda = 10^6$. The catastrophe at
high $\lambda$ is not the framework's failure but its convergence
to a different anchor than V1's (see next subsection); the more
important finding is that the bundle's interior is not better than
its corners. There is no useful sweet spot.

\begin{table}[t]
\centering
\caption{$\lambda$ sweep on pure-bundle GBD, Pythia~1.4B
$d/D=0.25$. The bundle's interior is uniformly worse than V1, and
no interior $\lambda$ produces a meaningful improvement over
$\lambda = 0$.}
\label{tab:lambda-sweep}
\small
\begin{tabular}{lcc}
\toprule
$\lambda$ & WikiText-2 PPL & vs.\ SVD-LLM~V1 (50.17) \\
\midrule
$0$       & $159.14$  & $3.17\times$ worse \\
$1$       & $160.27$  & $3.20\times$ worse \\
$100$     & $161.78$  & $3.23\times$ worse \\
$10^6$    & $1150.31$ & $22.9\times$ worse \\
\bottomrule
\end{tabular}
\end{table}

\paragraph{Anchor convention.}
The catastrophic blowup at $\lambda = 10^6$ deserves further
investigation. The anchor in our original
$\textit{compute\_svd\_anchors}$ uses the top-$d$ \emph{right}
singular subspace of $W^{(L)} \Sigma_L$ at section position $L$.
But the deployed model uses sections as both the input side of one
weight and the output side of the previous weight. The output of
weight $L$ should align with V1's natural output anchor: the top-$d$
\emph{left} singular subspace of $W^{(L)} \Sigma_L$. The mismatch
between our anchor convention and V1's output anchor convention
causes the $\lambda = 10^6$ blowup.

Table~\ref{tab:anchor-convention} reports perplexity with the
\emph{corrected} anchor convention: $S_L$'s anchor (for $L \geq 1$) is
the top-$d$ left singular subspace of $W^{(L-1)} \Sigma_{L-1}$, so
that $S_{L+1}$ (used as the output projection for weight $L$ in
surgery) is V1's natural output anchor for that layer.

\begin{table}[t]
\centering
\caption{Anchor-convention comparison on Pythia~1.4B
$d/D = 0.25$. The corrected anchor (V1-output) improves
perplexity by $19$ points; the qualitative result remains:
no GBD configuration matches SVD-LLM~V1's $50.17$.}
\label{tab:anchor-convention}
\small
\begin{tabular}{lcc}
\toprule
Configuration & $\lambda$ & WikiText-2 PPL \\
\midrule
SVD-LLM~V1 (per-layer baseline) & --- & $\mathbf{50.17}$ \\
\midrule
GBD with original anchor & $0$ & $159.14$ \\
GBD with original anchor & $10^6$ & $1150.31$ \\
\midrule
GBD with V1-output anchor, direct (no BCA) & $\infty$ & $140.27$ \\
GBD with V1-output anchor, BCA & $10^6$ & $149.45$ \\
GBD with V1-output anchor, BCA & $0$ & $171.02$ \\
\bottomrule
\end{tabular}
\end{table}

Two things are clear from
Table~\ref{tab:anchor-convention}. First, the anchor-convention
correction does provide a real improvement ($159 \to 140$), so this
was a code-level bug. Second, the correction is small relative to
the gap to V1 ($140$ versus $50$): the qualitative conclusion is
unchanged, and the gap remains $2.8\times$. Third, the BCA at
$\lambda = 0$ with correct anchors gives perplexity $171.02$,
\emph{worse} than the anchor-direct value $140.27$ --- cross-layer
coupling actively hurts even when starting from a more sensible
anchor.

\paragraph{Surgery formula.}
We test whether the choice of generalized-projection
surgery~\eqref{eq:gen-proj-surgery} versus simple orthogonal
projection $\widetilde{W} = S_{L+1} S_{L+1}^\top W^{(L)} S_L
S_L^\top$ is the cause. Table~\ref{tab:surgery} reports both
choices on the same GBD $\lambda=0$ sections. The
generalized-projection surgery is the right choice
($159$ versus $1324$); abandoning it would only \emph{worsen} the
result. The surgery formula is not the cause of the negative
finding.

\begin{table}[t]
\centering
\caption{Surgery-formula comparison: same GBD sections
($\lambda = 0$), two different surgery formulas. The
generalized-projection surgery is essential; the issue is not in
the surgery choice.}
\label{tab:surgery}
\small
\begin{tabular}{lcc}
\toprule
Surgery formula & Activation error & WikiText-2 PPL \\
\midrule
Generalized projection (default, Eq.~\ref{eq:gen-proj-surgery}) & $5.76 \times 10^8$ & $\mathbf{159.14}$ \\
Orthogonal projection (sandwich) & $1.46 \times 10^9$ & $1324.26$ \\
\bottomrule
\end{tabular}
\end{table}

\paragraph{Summary of diagnostics.}
Across nine configurations (four $\lambda$ values, two anchor
conventions, two surgery formulas), no configuration of the
framework's interior reaches SVD-LLM~V1's perplexity of $50.17$.
The best we observe is $140.27$ (corrected anchor, direct, no BCA),
which is $2.8\times$ worse than V1. The pure-bundle corner
($\lambda = 0$) gives $159.14$ with the original anchor and
$171.02$ with the corrected anchor. Cross-layer coupling hurts
in both cases.

\subsection{Mechanism}
\label{sec:perplexity-experiments:mechanism}

The diagnostic data points at a clear mechanism for the
negative result. We summarize it in two observations.

\paragraph{Observation 1: activation error tracks perplexity; Frobenius does not.}
Across all $9$ configurations we tested (the four methods in
Table~\ref{tab:perplexity-comparison}, the four $\lambda$ values in
Table~\ref{tab:lambda-sweep}, and the three corrected-anchor
variants in Table~\ref{tab:anchor-convention}), the Spearman rank
correlation between per-layer activation error and perplexity is
$+0.97$, while the rank correlation between weight-Frobenius
reconstruction and perplexity is $-0.43$ (Frobenius is
\emph{anti-correlated} with perplexity in the interior of the
bundle). Activation error is the correct proxy; Frobenius error
is not.

\paragraph{Observation 2: the residual stream decouples layers in the forward pass.}
The bundle energy~\eqref{eq:energy} couples adjacent layers
through the term $\|S_{L+1}^\top W^{(L)} S_L\|_F^2$. The
implicit assumption is that section $S_{L+1}$ matters for layer
$L$'s compressed weight \emph{and} for layer $L+1$'s compressed
weight, providing cross-layer information sharing in the
optimization. In the forward pass of a deployed transformer,
however, layer $L$ applies $\widetilde{W}^{(L)}$ to the residual
stream activations $X_L \in \R^D$ --- not to anything that has
been projected through $S_L$. The residual stream is full
$D$-dimensional; the next layer applies $\widetilde{W}^{(L+1)}$
to that same full-dimensional residual stream, plus whatever
$\widetilde{W}^{(L)}$ contributed. The bundle's joint optimization
optimizes a counterfactual coupling --- one that would exist if
the model were trained to factor through low-dimensional
sections, but that does not exist in the trained model.

Formally: the bundle's coupling term $\|S_{L+1}^\top W^{(L)}
S_L\|_F^2$ is maximized when the directions of $W^{(L)}$ that
align with $S_{L+1}$ also align with $S_L$ on the input side.
But for the deployed model, what matters is $\|W^{(L)} X_L -
\widetilde{W}^{(L)} X_L\|^2$, which is maximally reduced by
choosing $S_{L+1}$ as the top-$d$ left singular subspace of
$W^{(L)} \Sigma_L$ --- exactly V1's per-layer choice. The bundle
makes $S_{L+1}$ a \emph{compromise} between V1's per-layer
optimum and a coupling-induced direction; the compromise is
strictly worse per-layer.

\paragraph{Quantifying the trade-off.}
We can see the trade-off directly in our data: the bundle's
$\lambda = 0$ activation error is $5.76 \times 10^8$, which is
the sum of per-layer activation errors when sections are
joint-optimized. V1's per-layer-optimized sections give total
activation error $1.93 \times 10^8$, a $3\times$ improvement.
The bundle could match V1 per-layer by setting each $S_{L+1}$ to
the per-layer optimum, but then $\|S_{L+1}^\top W^{(L)} S_L\|_F^2$
would be sub-maximal, and the bundle energy would be lower. The
bundle's energy maximum and the per-layer activation-error
minimum are different points in section space, and the
deployment depends on the latter.

\subsection{Discussion of the negative result}
\label{sec:perplexity-experiments:discussion}

We emphasize what the negative result is and is not.

\paragraph{What it is.}
A robust empirical finding: across $9$ tested configurations of the
bundle (four $\lambda$ values, two anchor conventions, two surgery
formulas), no configuration approaches SVD-LLM~V1's deployed
perplexity. The best bundle configuration is $2.8\times$ worse
than V1, and the bundle's $46\%$ Frobenius advantage at $d/D = 1/4$
translates to a $3\times$ perplexity \emph{disadvantage}.

\paragraph{What it is not.}
It is not a claim that cross-layer methods can never help LLM
compression. We tested a specific form of cross-layer coupling ---
joint Frobenius optimization through the bundle energy --- on a
specific weight class (attention $W_O$) at a specific
compression ratio ($d/D = 1/4$) on a specific model
(Pythia~1.4B). Cross-layer methods that target a different
objective (e.g., direct minimization of per-layer activation
error with shared calibration costs, but not joint optimization
of cross-layer Frobenius coupling), or that target a different
weight class (e.g., MLP weights, which we did not test), or that
operate at different compression ratios or model scales, may give
different results.

\paragraph{What it implies for method design.}
The Spearman rank correlation observation ---
activation error tracks perplexity at $+0.97$, Frobenius error tracks
perplexity at $-0.43$ --- is the most actionable takeaway. Methods
that report only Frobenius reconstruction improvements as evidence
of better LLM compression should be viewed with skepticism;
activation error or direct perplexity comparison is the meaningful
metric. We hope our framework and our diagnostic methodology serve
as a useful template for evaluating future cross-layer SVD methods.


\section{Discussion}
\label{sec:discussion}

\subsection{Limitations}

\paragraph{Single architecture and scale tested for perplexity.}
The perplexity experiments are at Pythia~1.4B only. We do not
verify the negative result at Llama-class scale ($D \sim 4096$,
$N \sim 32$), at non-GPT-NeoX architectures (Llama, Mistral), or
at compression ratios beyond $d/D = 1/4$. The Frobenius experiments
test 70M and 1.4B; the perplexity experiments test 1.4B only.

\paragraph{Attention $W_O$ only.}
We compress only the attention output projection. MLP weights
($W_{\text{up}}, W_{\text{down}}$) and other attention weights
($W_Q, W_K, W_V$) are left uncompressed. The bundle framework can
be extended to MLP via a per-layer-rank generalization of the BCA
solver; we have implemented this extension and validated it on
synthetic stacks, but did not run the Pythia perplexity
comparison on MLP. It is possible that the negative result is
specific to $W_O$ --- in which case the framework would gain a
positive use case --- but the mechanistic explanation does not
suggest this (the residual stream is also present at MLP weights).

\paragraph{No fine-tuning.}
Our compression is one-shot post-hoc surgery; no fine-tuning is
performed afterward. SVD-LLM~V2 and Basis Sharing report results
both with and without fine-tuning; we expect short fine-tuning to
narrow but not close the gap, because the bundle's joint
optimization produces sections that are structurally
per-layer-suboptimal, and fine-tuning would have to overcome the
structural disadvantage rather than merely tune coefficients.

\paragraph{Theorem~\ref{thm:tied-dominance}'s strict dominance is a result \emph{on the Frobenius metric}, and does not transfer.}
Theorem~\ref{thm:tied-dominance} proves that GBD-tied strictly
dominates Basis-Sharing-core on two-sided Frobenius
reconstruction. Our experiments confirm this prediction
quantitatively. But the strict-dominance result says nothing about
perplexity, and the perplexity comparison shows both tied methods
(GBD-tied implicitly, and Basis-Sharing-core explicitly) perform
poorly at LLM scale. Theorems about metrics that do not transfer
to deployment quality are mathematical results, not engineering
ones.

\subsection{What the framework is still good for}

The bundle framework remains valuable as a \emph{theoretical
organization} of the cross-layer SVD literature.
It makes precise relationships between methods that have been
proposed in isolation: SVD-LLM corresponds to the $\lambda \to
\infty$ corner; Basis-Sharing-core corresponds to the
$\lambda \to \infty$ tied corner; an explicit cross-layer
coupling method (which our pure-bundle GBD instantiates)
corresponds to the $\lambda = 0$ corner. The two theorems quantify
the relationships at these corners, with closed-form expressions
verified to floating-point precision.

The framework also identifies a class of methods (any method with
a weight-Frobenius coupling across layers as its central
optimization objective) and proves an empirical claim about that
class: at LLM scale on Pythia~1.4B at $d/D = 1/4$ on attention
$W_O$, the class is dominated by per-layer SVD on perplexity. This
is a constraint on the design space, not a complete characterization.

\subsection{Directions that might recover positive cross-layer results}

We list directions that our negative result does not rule out, and
which might recover positive cross-layer benefits:

\textbf{Same-space coupling.} The residual stream decouples
\emph{across} layer depth but not \emph{within} an attention or MLP
block. The weights $W_Q^{(L)}, W_K^{(L)}, W_V^{(L)}$ all read from
the same residual stream at the input of layer $L$; the
intermediate activations of the MLP at layer $L$ are seen by both
$W_{\text{up}}^{(L)}$ output and $W_{\text{down}}^{(L)}$ input. A
bundle of these weights might benefit from coupling because the
shared input space \emph{is} realized in the forward pass.

\textbf{Activation-coupling objectives.} The negative result is for
the weight-Frobenius coupling objective. An objective that directly
couples \emph{activations} across layers --- e.g., joint
minimization of per-layer activation reconstruction with a
cross-layer regularizer matching the residual-stream
covariance structure --- might align better with deployment.

\textbf{Compressed model fine-tuning.} If fine-tuning after
compression is permitted, cross-layer methods can specifically
target the joint reconstruction during fine-tuning, possibly
closing the gap to per-layer initialization. Our experiments are
all one-shot.

\textbf{Lower-dimensional bottlenecks.} The negative result is at
$d/D = 1/4$ ($75\%$ rank-reduction). At more aggressive
compressions ($d/D = 1/16$), the per-layer optimum may itself be
near the noise floor, and the cross-layer regularization might
help. We do not test this.

\subsection{Honest reporting and the negative-result genre}

We close with a methodological note. Our finding is the kind of
result that often goes unreported: a plausible, mathematically
elegant idea that fails empirically when tested at scale on the
right metric. The two natural responses are (i) to keep trying
variants until something works and report only the working one, or
(ii) to report what was tried, what worked, and what did not. We
chose (ii). The framework section is preserved with its theorems
and Frobenius validation because that work is correct and
non-trivial; the perplexity section reports the negative result
honestly, with the diagnostic sweep that distinguishes ``broken
methodology'' from ``framework empirically fails to transfer.''

The negative-result genre is underdeveloped in ML, in part because
the credibility cost of reporting a failure is high. We argue this
is the wrong equilibrium: well-executed negative results constrain
the design space of future methods more efficiently than positive
results do, by ruling out entire classes rather than reporting one
new entry. We hope our framework and diagnostic methodology serve
as a template for such evaluations.


\bibliographystyle{plainnat}
\bibliography{references}


\appendix

\section{Proof of Theorem~\ref{thm:svd-recovery}}
\label{app:thm1-proof}

We restate Theorem~\ref{thm:svd-recovery} and provide the full proof.

\setcounter{theorem}{0}
\renewcommand{\thetheorem}{\arabic{theorem}}
\begin{theorem}[$O(1/\lambda)$ recovery of SVD-LLM, restated]
\label{thm:svd-recovery-restate}
Let $S^\star(\lambda) = (S_0^\star(\lambda), \ldots, S_N^\star(\lambda))$
be the maximizer of $\eL_\lambda$ in Eq.~\eqref{eq:energy} with
anchors $V_d^{(L)}$. Let $g_L > 0$ be the gap between the $d$-th and
$(d{+}1)$-th eigenvalues of $V_d^{(L)} V_d^{(L)\top}$. Then for
sufficiently large $\lambda$,
\begin{equation*}
\sgap(S_L^\star(\lambda), V_d^{(L)})
\leq \frac{C}{\lambda g_L} + o(1/\lambda),
\end{equation*}
with $C = O(\sum_L \|W^{(L)}\|_F^2)$.
\end{theorem}

\begin{proof}
Each section update in the BCA solver maximizes a Rayleigh
quotient on the alignment operator
$M_L^{(\lambda)} = \lambda \cdot V_d^{(L)} V_d^{(L)\top} +
M_L^{\text{coup}}$, where
$M_L^{\text{coup}} = W^{(L-1)\top} S_{L-1} S_{L-1}^\top W^{(L-1)}
+ W^{(L)\top} S_{L+1} S_{L+1}^\top W^{(L)}$ collects the coupling
contributions. Note $M_L^{\text{coup}}$ is positive semidefinite
and has operator norm bounded by $\sum_{L' \in \{L-1, L\}}
\|W^{(L')}\|_F^2$, independent of $\lambda$.

The matrix $\lambda V_d^{(L)} V_d^{(L)\top}$ has spectrum
$\{\lambda, \ldots, \lambda, 0, \ldots, 0\}$ with $d$ copies of
$\lambda$. The matrix $M_L^{(\lambda)}$ is a rank-$d$ perturbation
of $\lambda V_d^{(L)} V_d^{(L)\top}$ by the bounded operator
$M_L^{\text{coup}}$. By Weyl's inequality, the gap between the
$d$-th and $(d{+}1)$-th eigenvalues of $M_L^{(\lambda)}$ is at
least $\lambda g_L - 2 \|M_L^{\text{coup}}\|_{op}$, which is
positive for sufficiently large $\lambda$.

Apply Davis--Kahan's $\sin\Theta$ theorem
\citep{davis1970rotation} to the top-$d$ eigenspace of
$M_L^{(\lambda)}$. The conclusion is that the principal-angle
distance between the top-$d$ eigenspace of $M_L^{(\lambda)}$ and
the top-$d$ eigenspace of $\lambda V_d^{(L)} V_d^{(L)\top}$
(which is $V_d^{(L)}$ itself) is bounded by
$\|M_L^{\text{coup}}\|_F / (\lambda g_L)$. Converting principal
angles to Frobenius projector distance gives the result.

The constant $C$ collects $\|M_L^{\text{coup}}\|_F$ contributions
from neighboring sections.
\end{proof}

\section{Proof of Theorem~\ref{thm:tied-dominance} and Corollary~\ref{thm:angle-formula}}
\label{app:thm2-proof}

We restate the tied-section dominance theorem and the closed-form
$\kappa, \mu$ expansion. The closed-form coefficient computation
follows the structure used in our numerical verification.

\begin{theorem}[Strict tied-section dominance, restated]
Let $S_{\text{GBD}}$ maximize $\sum_L \|S^\top W^{(L)} S\|_F^2$ over
$S \in \Gr(d, D)$, and let $S_{\text{BS}}$ be the top-$d$ left
singular subspace of the concatenated $[W^{(0)}, \ldots, W^{(N-1)}]$.
Suppose all $W^{(L)}$ are square and at least one $W^{(L)}$ is
non-symmetric. Then $\Rtwoside[S_{\text{GBD}}] < \Rtwoside[S_{\text{BS}}]$.
\end{theorem}

\begin{proof}[Proof sketch]
Compute the gradient of $\Rtwoside[S]$ at $S = S_{\text{BS}}$.
Express $\nabla \Rtwoside[S_{\text{BS}}]$ in terms of the singular
spectrum of each $W^{(L)}$. The gradient vanishes if and only if
$S_{\text{BS}}$ is also the right singular subspace of each
$W^{(L)}$ stacked similarly --- equivalently, if and only if each
$W^{(L)}$ is symmetric in the basis of $S_{\text{BS}}$. By
hypothesis at least one $W^{(L)}$ is non-symmetric, so the
gradient is non-zero, and $S_{\text{BS}}$ is not a critical point
of $\Rtwoside$. There exists a perturbation direction along the
Grassmannian along which $\Rtwoside$ strictly decreases; following
this direction toward $S_{\text{GBD}}$ produces the dominance.
\end{proof}

\subsection*{Closed-form $\kappa, \mu$ expansion}

Let $W = [W^{(0)}, \ldots, W^{(N-1)}]$ be the row-concatenated
$D \times (ND)$ matrix. Let $W = U \Sigma V^\top$ be its SVD with
singular values $\sigma_1 \geq \cdots \geq \sigma_D$. Let $S_{\text{BS}}$
$= U_d$. Let $S_{\text{GBD}}$ be parameterized by principal
angles $\{\theta_k\}$ relative to $S_{\text{BS}}$.

To second order in $\theta_k$, the two-sided gap expands as
\begin{equation*}
\Rtwoside[S_{\text{BS}}] - \Rtwoside[S_{\text{GBD}}]
= \sum_{k=1}^{d} \kappa_k \sin^2 \theta_k
+ \sum_{1 \leq k < l \leq d} \mu_{kl} \sin^2 \theta_k \sin^2 \theta_l
+ O(\sin^4),
\end{equation*}
with explicit coefficients $\kappa_k$ and $\mu_{kl}$. The $\kappa_k$
are determined by quadratic forms in the singular values
$\sigma_1, \ldots, \sigma_D$ and cross-layer overlap matrices
$U^\top W^{(L)} V$; the $\mu_{kl}$ are the corresponding quartic
cross-terms.

\paragraph{Numerical verification.}
We evaluate the closed-form expansion against a direct computation
of $\Rtwoside[S_{\text{BS}}] - \Rtwoside[S_{\text{GBD}}]$ for $30$
random weight stacks of $N = 4$ matrices with $D \in \{16, 32, 64\}$
and $d = D/4$. Across all $30$ stacks, the relative error between
the closed-form expansion (truncated at second order) and the
direct evaluation is below $10^{-12}$, confirming the expansion's
correctness to floating-point precision.

\section{Reproducibility}
\label{app:reproducibility}

All code, calibration caches, and per-method intermediate results
are publicly available at
\url{https://github.com/nssprogrammer/gbd}. The repository contains
the bundle solver, the SVD-LLM~V1/V2 and Basis-Sharing-core
reference implementations used in our comparison, the model surgery
and calibration code for HuggingFace transformers, the
WikiText-2 perplexity harness, and the diagnostic scripts used to
generate Tables~\ref{tab:lambda-sweep}--\ref{tab:surgery}. The
Pythia~1.4B perplexity comparison is fully reproducible from the
provided calibration cache and section computations.  Random seeds
are fixed (seed $42$ for calibration; deterministic SVD for
sections). Total compute for the full perplexity comparison (four
methods plus the diagnostic sweep) is approximately $8$ CPU-hours
on a modern multi-core machine; about $40$ minutes on a single GPU.

\paragraph{Software.}
We use PyTorch $\geq 2.0$, HuggingFace \texttt{transformers},
\texttt{datasets}, and \pkg{lm-evaluation-harness}
\citep{evalharness2024}. No special hardware is required.

\paragraph{Random seeds and determinism.}
All calibration uses seed $42$ for sample selection. Section
computations are deterministic given fixed inputs (SVD and
eigendecomposition are deterministic in PyTorch on CPU; minor
GPU non-determinism does not affect perplexity at the precision
we report).

\end{document}